\documentclass[12pt]{article}

\usepackage{arxiv}

\usepackage[utf8]{inputenc} 
\usepackage[T1]{fontenc}    
\usepackage{url}     
\usepackage{booktabs}       
\usepackage{amsfonts}       
\usepackage{nicefrac}       
\usepackage{microtype}      
\usepackage{lipsum}		
\usepackage{graphicx}
\usepackage{doi}

\usepackage{amsfonts}
\usepackage{latexsym,amssymb,amsmath,amsthm,amsfonts,graphicx,float}
\usepackage{stix}
\usepackage{tikz}
\allowdisplaybreaks
\usepackage{array}
\usepackage{colortbl}
\usepackage{hhline}
\usepackage{soul}
\usepackage{caption}
\usepackage{subcaption}
\usepackage{verbatim}
\usepackage{amsfonts}
\usepackage{amsmath}
\usepackage{csquotes} 
\usepackage{booktabs}
\usepackage{dsfont}
\usepackage{xparse}
\usepackage{amsmath,amssymb}
\usepackage{graphicx}

\newcommand{\inner}[2]{\left< {#1}, {#2} \right>}
\newcommand{\innerE}[2]{\mathbf{E}\left[{#1} {#2}\right]}
\newcommand{\D}{\Delta^{\varphi}}
\newcommand{\h}{\mathcal{H}^{\D}} 
\newcommand{\hinv}{\Tilde{\mathcal{H}}^{\D}}

\newtheorem{theorem}{Theorem}
\newtheorem{definition}{Definition}
\newtheorem{example}{Example}
\newtheorem{remark}{Remark}
\newtheorem{proposition}{Proposition}
\newtheorem{corollary}{Corollary}

\title{The Hyperdimensional Transform: a Holographic Representation of Functions}

\date{} 					

\author{\href{https://orcid.org/0000-0003-4564-1672}{Pieter Dewulf$^{\hspace{1mm}\includegraphics[scale=0.06]{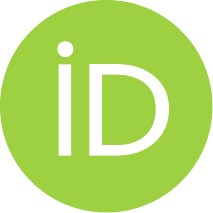}\hspace{1mm}}$},  \href{https://orcid.org/0000-0003-0903-6061}{Michiel Stock$^{\hspace{1mm}\includegraphics[scale=0.06]{orcid.pdf}\hspace{1mm}}$}, \href{https://orcid.org/0000-0002-3876-620X}{Bernard De Baets$^{\hspace{1mm}\includegraphics[scale=0.06]{orcid.pdf}\hspace{1mm}}$} \\
\texttt{pieter.dewulf@ugent.be}, \texttt{ michiel.stock@ugent.be}, \texttt{ bernard.debaets@ugent.be} \\
KERMIT, Department of Data Analysis and Mathematical Modelling \\
Ghent University}

\date{}

\renewcommand{\headeright}{}
\renewcommand{\undertitle}{}

\hypersetup{
pdftitle={The Hyperdimensional Transform: a Holographic Representation of Functions},
pdfauthor={Pieter Dewulf, Michiel Stock, Bernard De Baets},
pdfkeywords={Integral transforms, differential equations, hyperdimensional computing, vector symbolic architectures, machine learning, efficient computing},
}

\begin{document}
\maketitle

\begin{abstract}
Integral transforms are invaluable mathematical tools to map functions into spaces where they are easier to characterize.
We introduce the hyperdimensional transform as a new kind of integral transform. It converts square-integrable functions into noise-robust, holographic, high-dimensional representations called hyperdimensional vectors. The central idea is to approximate a function by a linear combination of random functions.
We formally introduce a set of stochastic, orthogonal basis functions and define the hyperdimensional transform and its inverse. We discuss general transform-related properties such as its uniqueness, approximation properties of the inverse transform, and the representation of integrals and derivatives.
The hyperdimensional transform offers a powerful, flexible framework that connects closely with other integral transforms, such as the Fourier, Laplace, and fuzzy transforms. 
Moreover, it provides theoretical foundations and new insights for the field of hyperdimensional computing, a computing paradigm that is rapidly gaining attention for efficient and explainable machine learning algorithms, with potential applications in statistical modelling and machine learning. In addition, we provide straightforward and easily understandable code, which can function as a tutorial and allows for the reproduction of the demonstrated examples, from computing the transform to solving differential equations.
\end{abstract}

\keywords{Integral transforms, differential equations, hyperdimensional computing, vector symbolic architectures, machine learning, efficient computing}

\section{Introduction}\label{sec:intro}
\subsection{Integral transforms}
In mathematics, various kinds of integral transformations (often simply called integral transforms, emphasizing the result of the transformation) exist that map functions from their original space into a new space, e.g., the Laplace transform, the Fourier transform, the wavelet transform, the fuzzy transform and the Z-transform, to name but a few~\cite{beerends2003fourier,debnath2014integral, perfilieva2001fuzzy,sundararajan2016discrete}. The underlying idea is that some problems may be solved more easily in the new space and that the solution in this new space can be mapped back (approximately) to the original space. For example, the Laplace transform is a well-known tool for solving differential equations; the Fourier transform is a tool for analyzing functions in the frequency domain; and the fuzzy transform can be used to work with noisy data and for data compression purposes, in addition to solving differential equations. An integral transform can generally be expressed as a mathematical operator $\mathcal{T}$, taking the following form:
\[
\left(\mathcal{T}f \right)(s) = \int_{x_1}^{x_2} f(x) K(x,s)\, {\mathrm d}x\,.
\]
Here, the function $f$ is transformed into a function $\mathcal{T}f$, and the type of transformation is specified by the domains of $f$ and $\mathcal{T}f$ and by the integral kernel $K(x,s)$, which can be seen as a family of basis functions. 
For example, the Laplace transform converts a function of a real variable $x$ into a function of a complex variable $s$, and the exponential basis functions $e^{-sx}$ determine the integral kernel. 
The Z-transform converts a discrete-time signal of a variable $x$ into a function of a complex variable $s$, with integral kernel $s^{-x}$. 
As a last example, the fuzzy transform converts a function $f$ of a real variable $x$ into a function $\mathcal{T}f$ with as domain a finite set, and the integral kernel is determined by a finite fuzzy partition $\{A_s(x)\mid s=1,2,\ldots,n\}$. Since the domain of $\mathcal{T}f$ is a finite set in this case, one can stack all the evaluations $\left(\mathcal{T}f \right)(s)$ in a vector and interpret the transformation as a function-to-vector transformation.

\subsection{Hyperdimensional computing}
Our work bridges the above-mentioned integral transforms with the field of hyperdimensional computing (HDC)~\cite{kanerva2009hyperdimensional, kleyko2023survey,kleyko2022survey}. Hyperdimensional computing, also known as vector symbolic architectures (VSA), is a highly interdisciplinary field with connections to computer science, electrical engineering, artificial intelligence, mathematics, and cognitive science~\cite{ kleyko2023survey,kleyko2022survey}. Especially in the field of machine learning and data science, hyperdimensional computing has recently witnessed growing interest and an increase in applications as an energy-efficient method~\cite{kleyko2023survey}.

The basic idea of HDC is that objects of any type can be represented by high-dimensional distributed representations, called hyperdimensional vectors. HDC algorithms rely on a set of key vector operations with specific algebraic properties: binding, superposition (also called bundling or aggregation), permutation, and similarity measurement. These operations allow for fast and robust computations. The exact algebraic operations depend on the chosen type of hyperdimensional vector. As hyperdimensional computing largely started as an empirical field in various areas, different types of hyperdimensional vectors have been described and used (e.g., bipolar, binary, ternary, sparse, real-valued, or complex-valued).
Nevertheless, the following four properties are assumed to be essential~\cite{kanerva2009hyperdimensional}:
\begin{itemize}
    \item[(i)] {\em Hyperdimensionality}: the vectors should have a large number of dimensions, e.g.~10,000 or more.
    \item[(ii)] {\em Robustness}: corruption of a small fraction of the vector should not result in a significant loss of information. The result of an HDC algorithm should be tolerant for such component failures. This robustness results from redundant representations.
    \item[(iii)] {\em Holistic or holographic representation}: information should not be locally stored but distributed `equally' over the entire vector. This is very different from the regular representation of data in computers, where specific bits have specific meanings.
    \item[(iv)] {\em Randomness}: vectors should be drawn randomly, typically with its elements independent and identically distributed.
\end{itemize}

These properties take inspiration from the functioning of the brain and allow for the implementation of various aspects of artificial intelligence such as memory, reasoning, and learning. For more details on the hyperdimensional space, we refer to~\cite{kanerva2009hyperdimensional,kleyko2023survey,kleyko2022survey}.

\subsection{Further outline}
This paper introduces a linear operator that transforms functions into hyperdimensional vectors, as defined by the four properties above. We are aware of one recent work where a similar idea of representing functions as hyperdimensional vectors was presented. In~\cite{frady2022computing}, the authors showed the analogy to kernel methods and used kernels compatible with the hyperdimensional ``binding'' operation to map kernel-decomposed functions into hyperdimensional space. In our work, the mapping into hyperdimensional space is more general. We present it as a formal integral transform.

Concretely, in Section~\ref{sec:encoding}, we first provide a concrete, formal approach to the representation of objects as hyperdimensional vectors. To that end, we introduce the function $\D: X \rightarrow \mathbb{R}^D$, called a normalized hyperdimensional encoding, that maps elements $x$ of a universe $X$ into the hyperdimensional space $\mathbb{R}^D$ with $D$ a large number. The components $\D_s(x)$, $s=1,2,\ldots,D$, of this vector-valued function may be seen as orthogonal basis functions, similar to the function $e^{-sx}$ in the Laplace transform
and $A_s(x)$ in the fuzzy transform. 

Section~\ref{sec:transform} introduces the hyperdimensional transform for square-integrable functions as the linear operator $\h : L^2(X) \rightarrow \mathbb{R}^D$. A noteworthy difference with many integral transforms is that this transform is not limited to functions with as domain a real interval but is defined for functions with as domain a more abstract universe $X$.

Also note that, although $D$ is assumed large, a function from the infinite-dimensional space $L^2(X)$ is transformed into a finite-dimensional vector space $\mathbb{R}^D$. Hence, the transform can thus only represent an approximation of the original function. This behaviour is also allowed, for example, for the fuzzy transform. 

The remainder of this work discusses various transform-related properties of the hyperdimensional transform, such as uniqueness (Section~\ref{sec:transform}), the inverse transform (Section~\ref{sec:inversetransform}) and its approximation qualities (Section~\ref{sec:approximationproperties}), and the representation of derivatives, integrals and inner products (Section~\ref{sec:integralsandderivatives}).
In Section~\ref{sec:multiplevariables}, we extend the theory to functions of multiple variables that live in different universes.

In Section~\ref{sec:differentialandintegralequations}, as an application, we illustrate how linear differential equations and linear integral equations can be naturally represented in the hyperdimensional space.
Finally, in Section~\ref{sec:othermethodsandfuturedirections}, we discuss close connections with other integral transforms. We indicate how they differ from
the new hyperdimensional transform and in which types of applications it might serve.

\section{Hyperdimensional encoding}
\label{sec:encoding}
In this section, we provide a concrete, formal approach to the representation of objects as hyperdimensional vectors. Given a universe $X$ that is endowed with a measure, we define a hyperdimensional encoding as a function based upon a stochastic process, mapping elements of $X$ to hyperdimensional vectors. We also introduce a notion of normalization. The corresponding normalized hyperdimensional encoding is the first step towards mapping functions belonging to $L^2(X)$ into a hyperdimensional space.

\begin{definition}\label{encoding}
    Let $(X, \mathcal{A}, \mu)$ be a finite measure space and $\{\Phi(x) \mid x \in X\}$ a stochastic process taking values in a bounded set $S \subset \mathbb{R}$. 
    \begin{itemize}
    \item[(i)] A function $n : X \rightarrow \mathbb{R}_{>0}$ that satisfies
    \begin{equation}
        \label{normalization}
        \int_{x' \in X} \frac{\innerE{ \Phi(x)}{ \Phi(x')}}{n(x)n(x')}{\mathrm d}\mu(x') = 1\,,\quad \text{ for all } x \in X
        \,,
    \end{equation}
    is called a normalization function of the stochastic process. We define the normalized stochastic process as $\left\{ \Delta^{\Phi}(x) := \frac{\Phi(x)}{n(x)} \mid  x \in X\right\}$.
    \item[(ii)]  Consider the vector-valued functions $\varphi : X \rightarrow \mathbb{R}^D$ and $\D : X \rightarrow \mathbb{R}^D$ defined by
    \[
    \varphi(x) := \left[ \varphi_1(x), \varphi_2(x),\ldots, \varphi_D(x) \right]^\mathrm{T}
    \]
    and
    \[
    \D(x) := \left[ \frac{\varphi_1(x)}{n(x)},  \frac{\varphi_2(x)}{n(x)},\ldots,  \frac{\varphi_D(x)}{n(x)} \right]^\mathrm{T}
    \,.
    \]
The $D$ components of these vectors are independent sample functions from the stochastic processes $\Phi$ and $\Delta^{\Phi}$, respectively. 
If $\Delta^{\varphi}$ is Bochner integrable w.r.t.~$\mu$, then we say that $\Delta^{\varphi}$ is a normalized hyperdimensional encoding of $X$ w.r.t.~the stochastic process~$\Phi$. The function $\varphi$ is called the unnormalized hyperdimensional encoding.
\end{itemize}
\end{definition}

\begin{remark}
    \label{bochnerVSlebesque}
    The Bochner integral can be seen as the Lebesgue integral for vector-valued mappings~\cite{bogachev2020real, mikusinski1978bochner}. With the integrand taking values in the vector space $\mathbb{R}^D$,
    integration w.r.t.~$\mu$ should be interpreted componentwise. In 
    Eq.~(\ref{normalization}), the Lebesgue and the Bochner integral interpretations coincide as the integrand takes values in $\mathbb{R}$.
\end{remark}

\begin{remark}
A normalized hyperdimensional encoding $\Delta^{\varphi}$ of $X$ w.r.t.~the stochastic process $\Phi$ can only be defined if a normalization function $n$ can be found. 
The existence of such a function for an arbitrary stochastic process on an arbitrary measure space $(X,\mathcal{A},\mu)$ is not known in general. Examples~\ref{intervalencoding}--\ref{discreteexample} show some (less) obvious solutions for the normalization function on different measure spaces and with different stochastic processes.
\end{remark}

The adjective $\textit{hyperdimensional}$ refers to the fact that the dimensionality $D$ is huge\footnote{$D$ should be `large enough' such that an inner product approximates its
expected values `close enough', which can be quantified using appropriate concentration bounds.}. Dimensionalities of 10,000 dimensions or more are fairly typical~\cite{kanerva2009hyperdimensional}. According to the law of large numbers~\cite{dekking2005modern}, we have
\[
\lim_{D \rightarrow \infty} \inner{\varphi(x)}{\varphi(x')} = \mathbf{E}\left[\Phi(x)\Phi(x')\right] 
\]
and
\[
\lim_{D \rightarrow \infty} \inner{\Delta^{\varphi}(x)}{\Delta^{\varphi}(x')} = \innerE{\Delta^{\Phi}(x)}{\Delta^{\Phi}(x')}
\,.
\]
In the left-hand sides, $\inner{\cdot}{\cdot}$ takes two vectors in $\mathbb{R}^D$ as arguments and represents the Euclidean inner product scaled with the dimensionality $D$.
The expected values
on the right-hand sides also represent inner products but between stochastic variables. We can write
\[
\inner{\Phi(x)}{\Phi(x')} := \mathbf{E}\left[\Phi(x)\Phi(x')\right]
\]
and
\[
\inner{\Delta^{\Phi}(x)}{\Delta^{\Phi}(x')} := \innerE{\Delta^{\Phi}(x)}{\Delta^{\Phi}(x')}
\,.
\]
Depending on the context, either the expected value or the inner product notation can be used.
By construction, a normalized encoding $\Delta^{\varphi}$ exhibits the properties of being robust, holistic, and random: each vector component is an independent random sample, while information is statistically encoded via high-dimensional inner products that approximate expected values. 

\begin{proposition}
    \label{commutationTint}
    Let $\Delta^{\varphi}$ be a normalized 
    hyperdimensional encoding of $X$, then 
    for all $x \in X$ we have
     \arraycolsep=2pt
     \begin{eqnarray*}
        \int_{x' \in X} \inner{ \Delta^{\varphi}(x)}{ \Delta^{\varphi}(x')}{\mathrm d}\mu(x') = \inner{ \Delta^{\varphi}(x)}{\int_{x' \in X}  \Delta^{\varphi}(x'){\mathrm d}\mu(x')} 
        \,.
    \end{eqnarray*}
\end{proposition}

\begin{proof}
    This result follows from the fact that for any bounded linear operator $\mathcal{T}$ between two Banach spaces and a function $f$ that is Bochner integrable, $\mathcal{T}f$ is Bochner integrable, and $\mathcal{T}\int_{x\in X}f(x)d\mu(x) = \int_{x\in X} \left(\mathcal{T}f\right)(x)d\mu(x)$. This follows directly from the definition of the Bochner integral~\cite{mikusinski1978bochner}.
\end{proof}

\begin{proposition}
    \label{limit}
    Let $\Delta^{\varphi}$ be a hyperdimensional encoding of $X$, then we have
    \arraycolsep=2pt
    \begin{eqnarray*}
      \lim_{D\rightarrow\infty} \int_{x' \in X} \inner{ \Delta^{\varphi}(x)}{ \Delta^{\varphi}(x')}{\mathrm d}\mu(x')
        &=& \int_{x' \in X} \lim_{D\rightarrow\infty} \inner{ \Delta^{\varphi}(x)}{ \Delta^{\varphi}(x')}{\mathrm d}\mu(x')\\
        &=&  \int_{x' \in X} \inner{ \Delta^{\Phi}(x)}{ \Delta^{\Phi}(x')}{\mathrm d}\mu(x') = 1\,.
        \end{eqnarray*}
\end{proposition}

\begin{proof}
    Due to the law of large numbers~\cite{dekking2005modern}, we have 
    \[ \lim_{D\rightarrow\infty} \inner{ \Delta^{\varphi}(x)}{ \Delta^{\varphi}(x')} = \inner{\Delta^{\Phi}(x)}{\Delta^{\Phi}(x')}\,.\]
Since $S$ is bounded,  $\inner{ \Delta^{\varphi}(x)}{ \Delta^{\varphi}(x')}$ is also bounded on our finite measure space. Hence, the first equality follows from the dominated convergence theorem~\cite{mikusinski1978bochner}. The last equality follows from the definition of the normalization function.
\end{proof}

We now give some examples of a hyperdimensional representation via a normalized hyperdimensional encoding. In practice, many tasks pertain to a universe $X$ that consists of real values. Intuitively, the inner product $\inner{\varphi(x)}{\varphi(x'))}$ can be related to the similarity between $x$ and $x'$ and should be a decreasing function of $|x-x'|$. 
Different kinds of such representations have been proposed in the field of hyperdimensional computing~\cite{kleyko2022survey}. Typically, some `range size' $\lambda$ is chosen such that $\inner{\varphi(x)}{\varphi(x')}\approx 0$ if $|x-x'|\geq \lambda$. If $|x-x'|<\lambda$, one uses some kind of interpolation: the number of shared entries in $\varphi(x)$ and $\varphi(x')$ increases with decreasing $|x-x'|$. Below, we give a concrete example of such a hyperdimensional representation.

\begin{example}
\label{intervalencoding}
Let $X=[a,b]\subset \mathbb{R}$, $\mathcal{A}$ be the set of all subintervals of $X$ and $\mu$ the Lebesgue measure expressing
the length of each subinterval. We define a stochastic process $\Phi$ taking values in $S=\{-1,1\}$ by the following properties:
\begin{itemize}
    \item[(i)] $\mathbf{E}\left[\Phi(x)\right]=0$, for all $x\in X$;
    \item[(ii)] $\innerE{\Phi(x)}{\Phi(x')} = \max\left(0, 1 - \frac{|x-x'|} {\lambda} \right)$, for all $x,x'\in X$, for some $\lambda>0$.
\end{itemize}
Note that, because of $(i)$, the quantity $\innerE{\Phi(x)}{\Phi(x')} = \inner{\Phi(x)}{\Phi(x')}$ in $(ii)$ represents the covariance between the stochastic variables $\Phi(x)$ and $\Phi(x')$.

A $D$-dimensional sample $\varphi$ of such a process can be constructed. For example, first,
select points $x_k$ with $k=0,1, ... ,n+1$ and $x_{k+1}-x_{k} = \lambda$, such that $x_1,...,x_n \in [a,b]$ and $x_0, x_{n+1} \notin [a,b]$.
Then, map all points $x_k$ into corresponding $D$-dimensional random vectors $r(x_k) \in \mathbb{R}^D$ via $D$ independent Rademacher variables (taking values in $\{-1,1\}$ with equal probability). If $x=x_k$, then assign $\varphi(x) = r(x_k)$; if $x_k < x < x_{k+1}$, then, for each component $\varphi_i(x)$, sample a switching point $t_i$ from the uniform distribution on $]x_k, x_{k+1}[$, and assign $\varphi_i(x)=\varphi_i(x_k)$ if $x<t_i$ and $\varphi_i(x)=\varphi_i(x_{k+1})$ if $x\geq t_i$. Each component of $\varphi(x)$ is thus a piecewise continuous function switching at most $n+1$ times 
between $-1$ and $1$ at random locations, such that the expectation of the inner product between $\varphi(x)$ and $\varphi(x')$ decreases linearly with $|x-x'|$ until they are expected to become uncorrelated at $|x-x'|\geq \lambda$. 
Indeed, $\lim_{D\rightarrow\infty} \inner{\varphi(x)}{\varphi(x')} = \max \left(0, 1 - \frac{|x-x'|} {\lambda} \right)$, which is invariant w.r.t.~translation of the chosen points $x_k$.

Besides a construction for taking a sample of the stochastic process, a normalization function is also required for a concrete normalized encoding $\Delta^{\varphi}$. Finding a solution for the normalization function $n(x)$ corresponds to solving the nonlinear integral equation 
\[1 = \int_{x \in [a,b]} \frac{\max \left( 0, 1-\frac{|x-x'|}{\lambda} \right)}{n(x)n(x')}{\mathrm d}\mu(x)
\,,
\]
which is a special case of the nonlinear integral equation of the Hammerstein type with a singular term at the origin. This type of equation has been shown to have a positive measurable solution. For a proof and conditions, we refer to~\cite{coclite2000positive}. In practice, an approximate solution to the Hammerstein equation is often constructed by the method of successive approximation~\cite{nadir2014numerical}. As an illustration for approximating the function $n(x)$, we use the interval $[0,1]$ and set $\lambda=1/4$. We choose 100 equidistant points to compute and evaluate our approximation for $n(x)$. As an initial guess, we set 
\[
    n_0(x) = \sqrt{\int_{x' \in [a,b]} \max \left( 0, 1-\frac{|x-x'|}{\lambda} \right) {\mathrm d}\mu(x')}
\] 
and in each $i$-th iteration, we compute the function 
\[
\Tilde{1}_i(x) = \int_{x' \in [a,b]} \frac{ \max \left( 0, 1-\frac{|x-x'| }{l} \right)}{n_{i}(x)n_{i}(x')} {\mathrm d}\mu(x')
\]
and update 
\[
n_{i+1}(x) = n_{i}(x)\sqrt{\Tilde{1}_{i}(x)}
\,.
\]
The left and right panels in Figure~\ref{fig:nx} show the functions $\Tilde{1}_i(x)$ and $n_i(x)$ through 10 iterations. 
\begin{figure}[!t]
\centering
\includegraphics[width=0.48\textwidth]{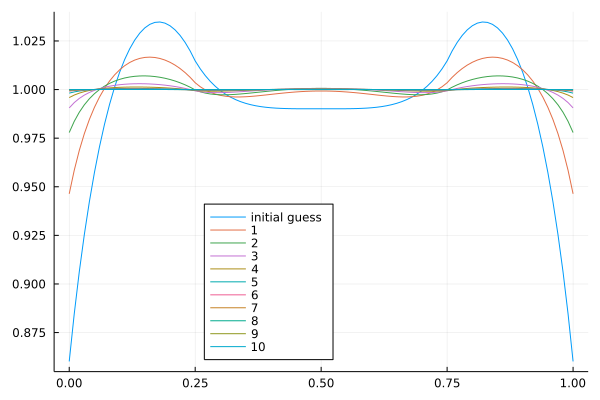}
\includegraphics[width=0.48\textwidth]{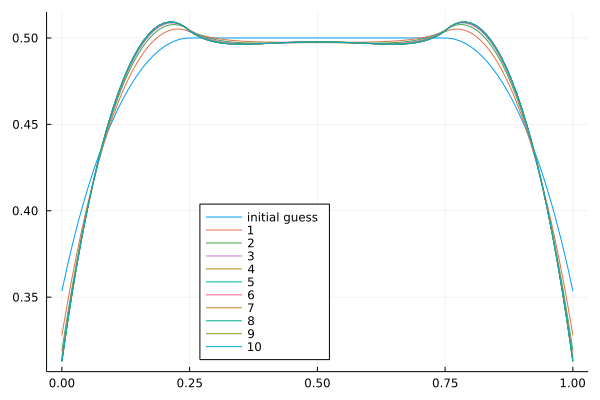}
\caption{The approximate solution of the Hammerstein equation converges, using an iterative approach. Left panel: $\Tilde{1}_i(x)$, right panel: $n_i(x)$.}
\label{fig:nx}
\end{figure}

Using the function $n_{10}(x)$ obtained after 10 iterations, in Figure~\ref{fig:basicfunctions}, we compare the normalized $\inner{\Delta^{\Phi}(x)}{\Delta^{\Phi}(x')}$ (right panel) to the unnormalized $\inner{\Phi(x)}{\Phi(x')}$ (left panel) in function of $x$ for a range of fixed $x'$. In the right panel, for each fixed $x'$, the area below $\inner{\Delta^{\Phi}(x)}{\Delta^{\Phi}(x')}$ is normalized to one.
\begin{figure}[!t]
\centering
\includegraphics[width=0.48\textwidth]{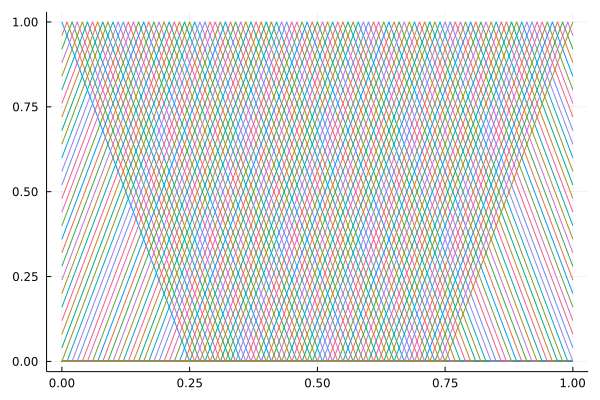}
\includegraphics[width=0.48\textwidth]{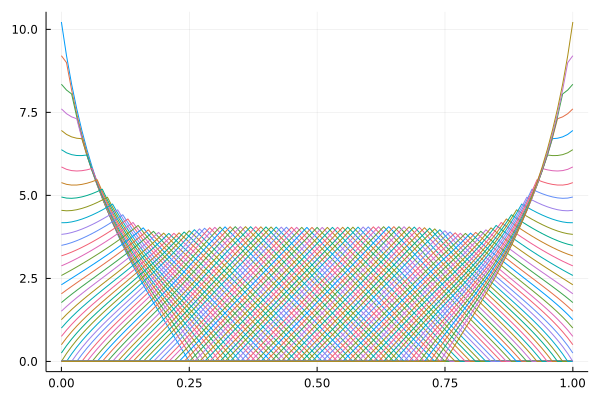}
\caption{The unnormalized $\inner{\Phi(x)}{\Phi(x')}$ (left pannel) and normalized $\inner{\Delta^{\Phi}(x)}{\Delta^{\Phi}(x')}$ (right panel) as a function of $x$ for a range of fixed $x'$. As a normalization function, the approximation $n_{10}(x)$ is used. For each fixed $x'$, the area below $\inner{\Delta^{\Phi}(x)}{\Delta^{\Phi}(x')}$ is approximately normalized to one.}
    \label{fig:basicfunctions}
\end{figure}
\end{example}

In the above example, one may recognize the function $\max \left(0, 1 - \frac{|x-x'|} {\lambda} \right)$ used as a basis function in the fuzzy transform~\cite{perfilieva2006fuzzy}.
\begin{example}
\label{exampleperiodic}
Consider the same measure space $(X,\mathcal{A}, \mu)$ as in Example~\ref{intervalencoding} and a similar stochastic process. Now we add periodic boundary conditions, {\em i.e.},
\begin{itemize}
    \item[(i)] $\mathbf{E}\left[\Phi(x)\right]=0$, for all $x\in X \,;$
    \item[(ii)] $\innerE{\Phi(x)}{\Phi(x')} = \max \left(0, 1 - \frac{d(x,x')} {\lambda} \right)$, for all $x,x'\in X$, with $\lambda=(b-a)/n$ and $n\geq 2$ an integer, and
\[
d(x,x') = \left\{
    \begin{array}{ll}
        |x-x'| &\text{, if }  |x-x'|\leq (b-a)/2 \\[.2cm]
       b-a-|x-x'| &\text{, if }  |x-x'|> (b-a)/2
    \end{array}
    \,.
\right.
\]
\end{itemize}
Note that $\innerE{\Phi(a)}{\Phi(b)}=\inner{\Phi(a)}{\Phi(b)}=1$.  

Indeed, a $D$-dimensional sample $\varphi$ of such a process can be constructed analogously to Example~\ref{intervalencoding}, with the difference that for the points $x_0,x_{n+1} \notin [a,b]$, we set $r(x_0)=r(x_n)$ and $r(x_{n+1})=r(x_1)$ and also, we set $\varphi(x_0+x)=\varphi(x_n+x)$ for $x\in\, ]0,\lambda[$.

We obtain a constant as a solution for the normalization function $n(x)$ since
    \[
        n^2(x) = \int_{x' \in [a,b]} \inner{ \Phi(x)}{ \Phi(x')} {\mathrm d}\mu(x') 
    = \lambda \,,
    \]
for all $x \in X$, and can define a concrete normalized hyperdimensional encoding as $\Delta^{\varphi}=\varphi/\sqrt{\lambda}$.
\end{example}

Note that, if there are no boundaries, the normalization function behaves as a constant. Similarly, in Example~\ref{intervalencoding}, away from the boundaries (w.r.t.\ the `range size' $\lambda=1/4$), the normalization function behaves as a constant.
To construct encodings for higher-dimensional spaces, 
we refer to Section~\ref{sec:multiplevariables}.
For completeness, we also illustrate the encoding of a finite set by means of the example below that assumes a simple structure.

\begin{example} \label{discreteexample}
    Let $X = U\times V \times W$ with $U,V,W$ discrete sets,  $\mathcal{A}$ the power set of $X$ and $\mu$ the counting measure. 
    We define a stochastic process $\Phi$ by the following properties:
    \begin{itemize}
        \item[(i)]  $\mathbf{E}\left[\Phi(x)\right]=0$, for all $x\in X \,;$
        \item[(ii)] $\innerE{\Phi(x)}{\Phi(x')} = \frac{1}{3} \left( \delta_{x_1,{x'}_1}+\delta_{x_2,{x'}_2}+\delta_{x_3,{x'}_3} \right)$, for all $x=(x_1,x_2,x_3),x'=(x'_1,x'_2,x'_3) \in X \,.$
    \end{itemize}
    Here, $\delta$ denotes the Kronecker delta function.
    
    A $D$-dimensional sample $\varphi$ of such a process can indeed be constructed. For example, first, map all elements $x_1\in U$, $x_2\in V$, $x_3\in W$ into corresponding $D$-dimensional random vectors $r_1(x_1), r_2(x_2), r_3(x_3) \in \mathbb{R}^D$ via $D$ Rademacher variables, taking values in $\{-1,1\}$ with equal probability. Then define $\varphi(x) =  \frac{1}{\sqrt{3}} \left( r_1(x_1)+ r_2(x_2)+r_3(x_3) \right)$, and, indeed $\lim_{D\rightarrow\infty} \inner{\varphi(x)}{\varphi(x')} = \frac{1}{3} \left( \delta_{x_1,{x'}_1}+\delta_{x_2,{x'}_2}+\delta_{x_3,{x'}_3} \right)$.
    As normalization function $n(x)$, we obtain a constant solution by counting:
    \arraycolsep=2pt
    \begin{eqnarray*}
        n^2(x) &=& \int_{x' \in X} \inner{ \Phi(x)}{ \Phi(x')} {\mathrm d}\mu(x') \\[.4cm]
        &=& \int_{x' \in X}  \frac{1}{3} \left( \delta_{x_1,{x'}_1}+\delta_{x_2,{x'}_2}+\delta_{x_3,{x'}_3} \right) {\mathrm d}\mu(x') \\[.4cm]
         &= & 1 + \frac{2}{3} \Bigl(\left(|U|-1\right)+\left(|V|-1\right)+\left(|W|-1\right) \Bigr)\\[.2cm]
        & &+  \frac{1}{3}\Bigl( (|U|-1)(|V|-1) + (|U|-1)(|W|-1) + (|V|-1)(|W|-1) \Bigr)\\[.2cm]
    &=& \frac{1}{3}\Bigl( |U||V| + |V||W| + |W||U| \Bigr) 
    \end{eqnarray*}
    for all $x \in X$. With the concrete construction for the sample $\varphi$ and the normalization constant $n$, we can define a normalized hyperdimensional encoding 
    as $\Delta^{\varphi}=\varphi/n$.

In case $U,V,W = \{0,1\}$, one may recognize the simple matching coefficient in $\innerE{\Phi(x)}{\Phi(x')} = \frac{1}{3} \left( \delta_{x_1,{x'}_1}+\delta_{x_2,{x'}_2}+\delta_{x_3,{x'}_3} \right)$, used for expressing the similarity between objects with binary attributes~\cite{de2009transitivity,verma2019new}.
\end{example}

In the field of hyperdimensional computing, approaches for constructing hyperdimensional representations have been described for many more universes $X$, representing different types of data structure such as graphs, images, sequences, symbols, sets, trees, and other structures~\cite{plate1995holographic,kleyko2022survey}. These approaches all have a random aspect in common. Our main contribution in this section is the formalization as a stochastic process with expected values and the notion of normalization, which is needed to formulate a proper transform in the next section.

\section{The hyperdimensional transform}
\label{sec:transform}
In this section, we use the normalized hyperdimensional encoding $\Delta^{\varphi}: X \rightarrow \mathbb{R}^D$ to construct the linear operator $\h$ that transforms functions from $L^2(X)$ into $\mathbb{R}^D$. The components $\Delta^{\varphi}_i$ that result from independent samples of a stochastic process will serve as orthogonal basis functions on which a function $f$ is projected.

We adhere to the following assumptions throughout this section: $(X, \mathcal{A}, \mu)$ is a finite measure space; $\{\Phi(x) \mid x \in X\}$ is a stochastic process taking values in a bounded set $S \subset \mathbb{R}$; and $\Delta^{\varphi}$ is a normalized hyperdimensional encoding of $X$ w.r.t.~the stochastic process $\Phi$. These are also the standing assumptions for the remainder of this work unless indicated differently.

\begin{definition}
\label{def:transform}
    The hyperdimensional transform w.r.t.~$\Delta^{\varphi}$ is defined by the linear operator $\mathcal{H}^{\Delta^{\varphi}} : L^2(X) \rightarrow \mathbb{R}^D $ as:
    \[
    F = \mathcal{H}^{\Delta^{\varphi}}f := \int_{x\in X} f(x)\Delta^{\varphi}(x){\mathrm d}\mu(x)    \,,
    \]
    and maps any real-valued function $f$ in $L^2(X)$ to a $D$-dimensional real vector~$F$. The product in the integrand
    is the product of the vector $\Delta^{\varphi}(x) \in \mathbb{R}^D$ and the scalar quantity $f(x) \in \mathbb{R}$. 
    The integral should again be interpreted as the Bochner integral w.r.t.\ the Lebesgue measure~$\mu$.
\end{definition}

Note that the operator $\mathcal{H}^{\Delta^{\varphi}}$ is linear, {\em i.e.},
     \[
        \mathcal{H}^{\Delta^{\varphi}}(\alpha f + \beta g) =  \alpha \mathcal{H}^{\Delta^{\varphi}}f +  \beta \mathcal{H}^{\Delta^{\varphi}}g
        \,,
    \]
     with $\alpha$ and $\beta$ two scalars and $f$ and $g$ two real-valued functions in $L^2(X)$. This linearity allows us to extend the hyperdimensional transform from real-valued functions to complex-valued ones, using $f = f_{\textit{real}} + \mathbf{i} f_{\textit{im}}$.

\begin{remark}    \label{finL1}
\noindent
\begin{itemize}
\item[(i)] The operator $\mathcal{H}^{\Delta^{\varphi}} : L^2(X) \rightarrow \mathbb{R}^D $ maps from one Hilbert space to another. Note that the Hilbert space $L^2(X)$ is infinite-dimensional, while $\mathbb{R}^D$ is finite-dimensional and we assume $D$ to be large.
\item[(ii)] One of the assumptions throughout this section is that the measure space $(X,\mathcal{A},\mu)$ is finite, in which case a function $f$ in $L^2(X)$ also belongs to $L^1(X)$ and thus is Lebesgue integrable. 
\item[(iii)]     Since the $i$-th component $\Delta^{\varphi}_i$ of the vector-valued function $\Delta^{\varphi}$) is bounded and Lebesgue integrable (see Definition~\ref{encoding} and Remark~\ref{bochnerVSlebesque}),
and $f$ is Lebesgue integrable (see (ii)), the product $f \Delta^{\varphi}_i$ is also Lebesgue integrable. Hence, the vector-valued function $f \Delta^{\varphi}$ is Bochner integrable and the transform is well-defined.
\end{itemize}
\end{remark}

\begin{remark}
\label{rem:distributions}
The definition of the hyperdimensional transform for functions can be extended with one for measures.
Let $\mathcal{M}(X,\mathcal{A})$ denote the space of all real-valued measures on the measurable space $(X,\mathcal{A})$.
Then the hyperdimensional transform of a measure $\mu' \in \mathcal{M}(X,\mathcal{A})$ w.r.t.\ $\Delta^{\varphi}$ can be defined by the linear operator $\mathcal{H}^{\Delta^{\varphi}}_* : \mathcal{M}(X,\mathcal{A}) \rightarrow \mathbb{R}^D $ as:
\[
    M' = \mathcal{H}^{\Delta^{\varphi}}_*\mu' = \int_{x\in X}\Delta^{\varphi}(x){\mathrm d}\mu'(x)\,.
\]
Instead of weighing the integration with a function $f$, now a measure $\mu'$ is used. 
This extension allows for the interpretation $\Delta^{\varphi}(x) = \mathcal{H}^{\Delta^{\varphi}}_*\delta_x$ with $\delta_x$ the Dirac measure peaked at $x$. If $\mu'=\mu$, then $ \mathcal{H}^{\Delta^{\varphi}}_* \mu = \mathcal{H}^{\Delta^{\varphi}}1_X$, with $1_X$ the simple function mapping all elements of $X$ to $1$.
\end{remark}


The following theorem expresses that the hyperdimensional transform is unique, {\em i.e.}, the transform is injective if the function $\inner{\D(\cdot)}{\D(\cdot)} : X \times X \rightarrow \mathbb{R}$ is a strictly positive definite kernel function. For completeness, we first recall the definition of 
such kernel function~\cite{muandet2017kernel}.

\begin{definition}
\label{def:posdefkernel}
    A function $k : X \times X \rightarrow \mathbb{R}$ is a positive definite kernel function if it is symmetric, i.e., $k(x^1,x^2)=k(x^2,x^1)$, and any Gram matrix is positive definite, i.e.,
    \[
\sum_{i=1}^n \sum_{j=1}^n c_i c_j k(x^i, x^j) \geq 0\,,
\]
for any $n\in \mathbb{N}$, any $x^1,\ldots, x^n \in X$ and any $c_1,\ldots, c_n \in \mathbb{R}$~\cite{muandet2017kernel,shawe2004kernel}. The function is said to be strictly positive definite if the equality 
\[
\sum_{i=1}^n \sum_{j=1}^n c_i c_j k(x^i, x^j) = 0
\]
implies $c_1=c_2=\ldots=c_n=0$.
\end{definition}

This definition is equivalent to saying that the eigenvalues of any Gram matrix of a positive definite kernel function $k$, {\em i.e.},~any $n\times n$ matrix $K$ with $K_{ij}=k(x^i,x^j)$ for any $x^1,\ldots,x^n$ and any $n\in\mathbb{N}$, are non-negative. For a strictly positive definite kernel function, the eigenvalues of the Gram matrix must be strictly positive~\cite{muandet2017kernel,shawe2004kernel}.

\begin{theorem}
       Let $F=\h f$ and $G=\h g$ be the hyperdimensional transforms of $f, g \in L^2(X)$. If the function $\inner{\D(\cdot)}{\D(\cdot)} : X \times X \rightarrow \mathbb{R}$ is a strictly positive definite kernel function, then $F=G$ implies $f=g$.
\end{theorem}

\begin{proof}
Note that 
    \[
    F - G = \int_{x\in X} \left(f(x)-g(x)\right)\Delta^{\varphi}(x){\mathrm d}\mu(x)
    \,.
    \]
Using Proposition~\ref{commutationTint}, we have 
    \arraycolsep=2pt
    \begin{eqnarray*}
    \lVert F-G \rVert^2
        &=& \inner{F-G}{F-G} \\
        &=&  \int_{x\in X} \int_{x'\in X} \big(f(x)-g(x) \big) \big(f(x')-g(x')\big) \inner{\D(x)}{\D(x')} {\mathrm d}\mu(x){\mathrm d}\mu(x')
        \,.
        \end{eqnarray*}
The strictly positive definiteness of $\inner{\D(\cdot)}{\D(\cdot)} : X \times X \rightarrow \mathbb{R}$ then implies that if 
$\lVert F-G \rVert^2=0$, then also $f=g$~\cite{muandet2017kernel,shawe2004kernel}. 
\end{proof}

\begin{example}
    The function $\inner{\Phi(x)}{\Phi(x')} = \max \left( 0, 1-\frac{|x-x'|}{\lambda} \right)$, for $x,x'\in [a,b]$, introduced in Example~\ref{intervalencoding}, is positive definite.
    Indeed, the Fourier transform of the non-negative function $\frac{2-2\cos(\omega)}{\omega^2}$ in the frequency domain is proportional to the function $\max \left( 0, 1-|x-x'| \right)$ with $\lambda=1$ in the spatial domain, such that positive definiteness follows from Bochner's theorem~\cite{muandet2017kernel}.
    To obtain a unique transform,
    a strictly positive definite function can always be constructed by using a modified stochastic process $\Phi_{\varepsilon}$ with a small probability $\varepsilon$ for which the outcome of the stochastic process is a completely random function, such that
    \[
    \inner{\Phi_{\varepsilon}(x)}{\Phi_\varepsilon(x')} = (1-\varepsilon) \inner{\Phi(x)}{\Phi(x')} + \varepsilon \delta_{x,x'}
    \,.
    \]
    Since $\sum_{i=1}^n \sum_{j=1}^n c_i c_j \delta_{x^i,x^j} = 0$ implies $c_1=\ldots=c_n=0$, even the slightest $\varepsilon$ turns a positive definite function  $\inner{\Phi(x)}{\Phi(x')}$ into a strictly positive definite function $ \inner{\Phi_{\varepsilon}(x)}{\Phi_\varepsilon(x')}$.
\end{example}

\section{The inverse hyperdimensional transform}
\label{sec:inversetransform}

This section introduces the inverse hyperdimensional transform, a linear operator $\hinv$ that transforms vectors in $\mathbb{R}^D$ back into $L^2(X)$. The back-transformed function $\Tilde{f}=\hinv \h f$ is to be understood as an approximation of the original function and not its exact recovery, though it can be an arbitrarily close approximation.

\begin{definition}
\label{def:inversetransform}
 
 The inverse hyperdimensional transform w.r.t.~$\Delta^{\varphi}$ is defined by the linear operator $\Tilde{\mathcal{H}}^{\Delta^{\varphi}} : \mathbb{R}^D \rightarrow L^2(X)$ as:
    \[
        \Tilde{\mathcal{H}}^{\Delta^{\varphi}}F := \inner{F}{\Delta^{\varphi}(\cdot)}\,,
    \]
    with function evaluation
    \[
       \left(\Tilde{\mathcal{H}}^{\Delta^{\varphi}}F\right)(x) = \inner{F}{\Delta^{\varphi}(x)}
       \,,
    \]
    and maps any $D$-dimensional vector $F$ to a real-valued function in $L^2(X)$.
\end{definition}

\begin{remark}
    The function $\Tilde{f} = \Tilde{\mathcal{H}}^{\Delta^{\varphi}}F$ is indeed an element of $L^2(X)$ for all $F\in \mathbb{R}^D$. Since $\Delta^{\varphi}$ is Bochner integrable, we have that $\Tilde{f} = \inner{F}{\D(\cdot)}$ is Lebesque integrable (see Remark~\ref{commutationTint} and the proof of Proposition~\ref{commutationTint}) and since $\Tilde{f}$ is bounded because $\Delta^{\varphi}$ is bounded, $\Tilde{f}$ is also square Lebesgue integrable.
\end{remark}

\begin{remark}
    The back-transformed function $\Tilde{f} = \Tilde{\mathcal{H}}^{\Delta^{\varphi}} {\mathcal{H}}^{\Delta^{\varphi}} f$ does not yield the original function $f$ but
    an approximation thereof. More specifically, based on Proposition~\ref{commutationTint}, we have: 
    \arraycolsep=2pt
    \begin{eqnarray*}
        \Tilde{f}(x) &=& \left( \hinv \h f \right)(x)\\
           &=& \int_{x'\in X} f(x')\inner{\Delta^{\varphi}(x)}{\Delta^{\varphi}(x')}{\mathrm d}\mu(x')
        \,,
    \end{eqnarray*}
        which can be interpreted as a smoothened version of the original function, according to the kernel function $\inner{\Delta^{\varphi}(\cdot)}{\Delta^{\varphi}(\cdot)}$.
\end{remark}

\begin{remark}
    \label{function1exact}
    With $1_X$ the simple function mapping all elements of $X$ to $1$, we have:
    \[
        \Tilde{1}_X(x) = \left(\hinv \h 1_X \right)(x) = \int_{x'\in X} \inner{\Delta^{\varphi}(x)}{\Delta^{\varphi}(x')}{\mathrm d}\mu(x')
        \,,
    \]
    which converges to $1_X$ for $D\rightarrow\infty$, expressing the normalization requirement of $\Delta^{\varphi}$ (Proposition~\ref{limit}, Definition~\ref{encoding}).
\end{remark}

\begin{example}
\label{exampledifferentD}
    We perform a brief experiment in which $\Tilde{f} = \hinv \h f$ can be compared to $f$.
    We use the normalized encoding of an interval introduced in Example~\ref{intervalencoding}, and set the interval $X=[0,1]$, $\lambda=1/20$, the normalization function $n=n_{10}$, obtained via 10 iterations, and $f : x \mapsto x \sin(10x)$. The results for $D=5000$, $10,000$, and $50,000$ are shown in Figure~\ref{examplerecovery}.

    \begin{figure}
        \centering        
        \includegraphics[width=0.65\textwidth]{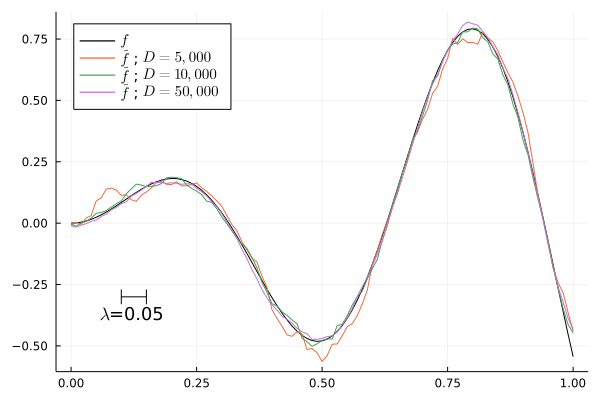}
        \caption{Comparison of the function $f: x \mapsto x\sin(10x)$ and $\Tilde{f}=\hinv \h f$ for different dimensionalities $D$, using the normalized hyperdimensional encoding $\D$ from Example~\ref{intervalencoding} with $\lambda=1/20$.}
        \label{examplerecovery}
    \end{figure}
\end{example}

\section{Approximation properties of the inverse hyperdimensional transform}
\label{sec:approximationproperties}
In this section, we additionally assume that $(X,d)$ is a metric space such that we can speak of continuous functions.
We describe some approximation properties of $\hinv \h f$ in the limit $D\rightarrow\infty$. Following the law of large numbers~\cite{dekking2005modern}, this limit is the expected value of $\hinv \h f$. Note that the expected absolute difference
of $\inner{\D(x)}{\D(x')}$ and its limit $\inner{\Delta^{\Phi}(x)}{\Delta^{\Phi}(x')}$ for $D\rightarrow \infty$ scale as $1/\sqrt{D}$. First, we introduce the notion of a length scale of an encoding.

\begin{definition}
\label{def:lengthscale}
    Let $(X,d)$ be a metric space and $\{\Phi_l(x) \mid x \in X\}$ a stochastic process, parameterized by $l\in\mathbb{R}_{>0}$, taking values in a bounded set $S \subset \mathbb{R}$, and let $\Delta^{\varphi_l}$ be a normalized hyperdimensional encoding of $X$ w.r.t.~the stochastic process $\Phi_l$. We say that the parameter $l$ is a length scale if for all $x,x'\in X$
    \begin{itemize}
        \item[(i)] $\innerE{\Delta^{\Phi_l}(x)}{\Delta^{\Phi_l}(x')} >  0$, if $d(x,x')< l$;
        \item[(ii)] $\innerE{\Delta^{\Phi_l}(x)}{\Delta^{\Phi_l}(x')} =  0$, if $d(x,x')\geq l$.
    \end{itemize}
    Here, note that $\innerE{\Delta^{\Phi_l}(x)}{\Delta^{\Phi_l}(x')} = \inner{\Delta^{\Phi_l}(x)}{\Delta^{\Phi_l}(x')} = \lim_{D\rightarrow\infty} \inner{\Delta^{\varphi_l}(x)}{\Delta^{\varphi_l}(x')}$.
  \end{definition}    
  
The length scale $l$ thus is a parameter that expresses over which distance the random variables $\Delta^{\Phi_l}(x)$ and $\Delta^{\Phi_l}(x')$ can be (positively) correlated.

\begin{theorem}\label{arbitrarilyclose}
Let $\Delta^{\varphi_l}$ be a normalized hyperdimensional encoding of $X$ parameterized by a length scale $l\in\mathbb{R}_{>0}$.
Let $f:X \rightarrow\mathbb{R}$ be a function in $L^2(X)$ that is continuous at $x \in X$. Then, for any $\varepsilon > 0$, there exists a length scale $l>0$ 
such that 
    \[
    |f(x) - \tilde{f}(x)| \leq \varepsilon \,,
    \]
    with $\tilde{f} = \lim_{D\rightarrow \infty} \Tilde{\mathcal{H}}^{\Delta^{\varphi_l}}\mathcal{H}^{\Delta^{\varphi_l}}f$.
\end{theorem}
    
    \begin{proof}
    Using $\int_{x' \in X} \inner{\Delta^{\Phi_l}(x)}{\Delta^{\Phi_l}(x')} {\mathrm d}\mu(x') = 1$ and Propositions~\ref{commutationTint} and \ref{limit}, we have
    \arraycolsep=2pt
    \begin{eqnarray*}
     \left|f(x) - \tilde{f}(x)\right| 
     && =
        \left|
        \int_{x'\in X} (f(x)-f(x')) \inner{\Delta^{\Phi_l}(x)}{\Delta^{\Phi_l}(x')} {\mathrm d}\mu(x')
         \right| \\[.2cm]
         && \leq
        \int_{x'\in X} \left|f(x)-f(x')\right| \inner{\Delta^{\Phi_l}(x)}{\Delta^{\Phi_l}(x')} {\mathrm d}\mu(x')\\[.2cm]
    &&  \leq 
        \max_{x'\in X, d(x,x')\leq l} \left|f(x) - f(x')\right|
        \,.
    \end{eqnarray*}
    Due to the continuity of $f$ at $x$, for every $\varepsilon > 0$, there exists an $l >0$ such that $|f(x) - f(x')|<\varepsilon$ if $d(x,x')<l$.
\end{proof}

\begin{remark}
    Note that any piecewise continuous function on a real interval $X=[a, b]$ with a finite number of jump discontinuities can
    also be approximated arbitrarily well 
    on the entire domain. Assume that $f$ is continuous on subintervals $X^1, \ldots, X^n$ that form a partition of $X$.
    Now, a normalized hyperdimensional encoding $\D$ of $X$ can be constructed in terms of independent encodings ${\D}^{\,i}$ that each map subinterval $X^i$ to $\mathbb{R}^D$, {\em i.e.},  for $D\rightarrow \infty$, $\inner{{\D}^{\,i}(x)}{{\D}^{\,j}(x')}=0$ for $x,x' \in X^i, X^j$ and $i\neq j$. Analogous to the proof of Theorem~\ref{arbitrarilyclose}, one can easily show that $\Tilde{f} = \hinv \h f$ approximates the piecewise continuous function arbitrarily well by approximating the continuous functions on the different subintervals via independent encodings.
\end{remark}

\begin{example}
    Recall that in both  Examples~\ref{intervalencoding} and~\ref{exampleperiodic},
    it holds that $\inner{\Phi(x)}{\Phi(x')} =$ $ \max \left( 0,1-\frac{d(x,x')}{\lambda} \right)$. One easily verifies that $\lambda$ satisfies the requirements of a length scale $l$ with the following choice of metric $d$: 
    \begin{itemize}
        \item[(i)] Example~\ref{intervalencoding}: 
        $d(x,x')=|x-x'|  \,;$
        \item[(ii)] Example~\ref{exampleperiodic}: 
                \[  d(x,x') = \left\{
                    \begin{array}{ll}
                        |x-x'|    &\text{, if }  |x-x'|\leq (b-a)/2 \\[.2cm]
                       b-a-|x-x'| &\text{, if }  |x-x'|> (b-a)/2
                    \end{array}
                \right.
                \,.
                \]
        \end{itemize}
    Following Theorem~\ref{arbitrarilyclose}, the normalized hyperdimensional encodings in these examples thus allow for approximating any continuous function arbitrarily well in the limit of $D\rightarrow \infty$.
\end{example}
Also functions of discrete variables can be approximated arbitrarily close, as they are always continuous. One can always define an encoding parameterized by a length scale $l$ such that each element is only correlated to itself for $l$ approaching $0$.
As an example, we next extend Example~\ref{discreteexample} by including a length scale $l$ such that the requirements of the definition of a length scale and of Theorem~\ref{arbitrarilyclose} are fulfilled.

\begin{example}
   We define the metric 
   \[
   d(x,x') = 1 - \frac{\delta_{x_1,x'_1} + \delta_{x_2,x'_2} + \delta_{x_3,x'_3}}{3}
   \,,
   \]
   taking only values $0$, $1/3$, $2/3$ and $1$ on the discrete set $X$ from Example~\ref{discreteexample}. We define a modified stochastic process $\Phi_l$ parameterized by $l$ by
\begin{itemize}
        \item[(i)]  $\mathbf{E}\left[\Phi(x)\right]=0$, for all $x\in X \,;$
        \item[(ii)]
        \[
        \innerE{\Phi_l(x)}{\Phi_l(x')} = \left\{
    \begin{array}{ll}
        \delta_{x_1,{x'}_1} \delta_{x_2,{x'}_2}  \delta_{x_3,{x'}_3} &\text{, if }  l<1/3 \\[.2cm]
        \frac{1}{3} \left( \delta_{x_1,{x'}_1}\delta_{x_2,{x'}_2} +\delta_{x_2,{x'}_2}\delta_{x_3,{x'}_3} +\delta_{x_3,{x'}_3}\delta_{x_1,{x'}_1} \right)  &\text{, if }  1/3\leq l < 2/3 \\[.2cm]
        \frac{1}{3} \left(\delta_{x_1,{x'}_1} + \delta_{x_2,{x'}_2}  + \delta_{x_3,{x'}_3} \right) &\text{, if }  2/3 \leq l \\[.2cm]
    \end{array}
\right.  
\,,
\] 
for all $x=(x_1,x_2,x_3)\in X$ and $x'=(x'_1,x'_2,x'_3) \in X$.
    \end{itemize}
A sample can be constructed for all $x=(x_1,x_2,x_3) \in X$ as
\[
\varphi_l(x) = \left\{
    \begin{array}{ll}
        r_1(x_1)r_2(x_2)r_3(x_3) &\text{, if }  l<1/3 \\[.2cm]
        \frac{1}{\sqrt{3}} \left(r_1(x_1)r_2(x_2) + r_2(x_2)r_3(x_3) + r_3(x_3)r_2(x_1) \right)  &\text{, if }  1/3\leq l < 2/3 \\[.2cm] 
        \frac{1}{\sqrt{3}} \left( r(x_1)+r(x_2)+r(x_3) \right)  &\text{, if }  2/3 \leq l \\[.2cm]
    \end{array}
    \,.
\right.
\]  
For $l\geq 2/3$, the stochastic process is unchanged compared to Example~\ref{discreteexample}. For $l<1/3$, the encoding of each element of $X$ is only correlated to itself. One can compute the normalization constants for every $l$ and verify that $l$ is a length scale according to Definition~\ref{def:lengthscale}: random variables are positively correlated if the distance is smaller than $l$ and uncorrelated else.
\end{example}

Now, we provide an indication of the speed of convergence with length scale $l$ when $X\subset \mathbb{R}$ is a real interval, again assuming the limit $D\rightarrow \infty$.

\begin{theorem}
\label{quadraticapprox}
    Let $X=[a,b] \subset \mathbb{R}$ and $\Delta^{\varphi_l}$ a normalized hyperdimensional encoding parametrized by a length scale $l\in\mathbb{R_{>0}}$.
    Let $f$ be twice continuously differentiable
    and $\Tilde{f} = \lim_{D\rightarrow\infty} {\Tilde{\mathcal{H}}^{\Delta^{\varphi_l}}} \mathcal{H}^{\Delta^{\varphi_l}} f$, then
    \[
    \tilde{f}(x) = f(x) + O(l^2)
    \,.
    \]
\end{theorem}

\begin{proof}
        The proof is given for $l<b-a$ and $x \in [a+l, b-l]$, and is analogous when closer to the boundaries.
    In the limit $D\rightarrow\infty$, we have
    \[
    \tilde{f}(x) = \int_{x'\in[a,b]} f(x') \inner{\Delta^{\Phi}(x)}{\Delta^{\Phi}(x')} {\mathrm d}\mu(x')
    \,.
    \]
    Using the trapezium rule for a twice continuously differentiable function with the three points $x-l$, $x$, $x+l$, we have
    \arraycolsep=2pt
    \begin{eqnarray*}
          \tilde{f}(x) = l \bigg[ \frac{f(x-l)\inner{\Delta^{\Phi}(x)}{\Delta^{\Phi}(x-l)}}{2}  + \frac{f(x+l)\inner{\Delta^{\Phi}(x)}{\Delta^{\Phi}(x+l)}}{2}   + f(x)\inner{\Delta^{\Phi}(x)}{\Delta^{\Phi}(x)} \bigg] + O(l^2) \,,
    \end{eqnarray*}
and using $\inner{\Delta^{\Phi_l}(x)}{\Delta^{\Phi_l}(x')}=0$ if $|x-x'|\geq l$, we have
    \[
    \tilde{f}(x) = l f(x)\inner{\Delta^{\Phi}(x)}{\Delta^{\Phi}(x)} + O(l^2)
    \,.
    \]
    Similarly, using $\int_{x' \in [a,b]} \inner{\Delta^{\Phi_l}(x)}{\Delta^{\Phi_l}(x')} {\mathrm d}\mu(x')=1$, we have
    \[
    f(x) = \int_{[a,b]} f(x) \inner{\Delta^{\Phi_l}(x)}{\Delta^{\Phi_l}(x')} {\mathrm d}\mu(x)'
    \,,
    \]
    such that, again  using the trapezium rule,
    \[
    f(x) = l f(x)\inner{\Delta^{\Phi}(x)}{\Delta^{\Phi}(x)} + O(l^2)
    \,,
    \]
    and thus $\tilde{f}(x) = f(x) + O(l^2)$.
\end{proof}

\begin{example}
    Consider the settings of Example~\ref{exampledifferentD}, but instead of varying $D$, we set $D$ large (i.e., at 50,000) and vary $\lambda$ which acts as a length scale $l$ according to Theorems~\ref{arbitrarilyclose} and~\ref{quadraticapprox}. The approximated function for different length scales is shown in Figure~\ref{examplerecoverydifferentl}.
    \begin{figure}
        \centering        
        \includegraphics[width=0.65\textwidth]{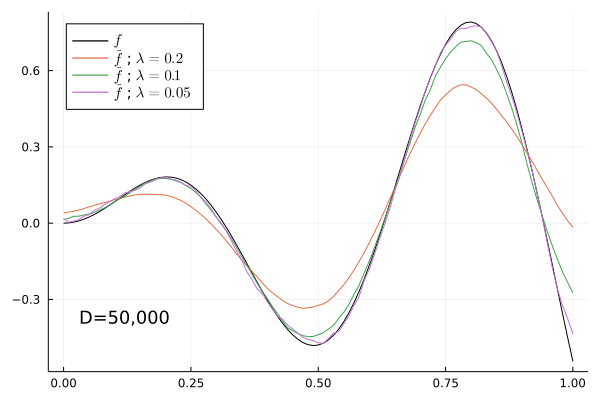}
        \caption{Comparison of the function $f: x \mapsto x\sin(10x)$ and $\Tilde{f}=\hinv \h f$ for different length scales, using the normalized hyperdimensional encoding $\D$ from Example~\ref{intervalencoding} with $D=50,000$. The larger the length scale, the smoother the approximating function $\Tilde{f}$. The smaller the length scale, the closer the approximation is to the original function.}
        \label{examplerecoverydifferentl}
    \end{figure}
\end{example}

\section{Integrals and derivatives}
\label{sec:integralsandderivatives}

In this section, we describe how integrals and derivatives of functions can be expressed in terms of their hyperdimensional transforms. First, we consider integrals, for which no additional assumptions are needed. 

\begin{theorem}
    \label{innerproduct}
     Let $F=\h f$ and $G=\h g$ be the hyperdimensional transforms of $f, g \in L^2(X)$, then
    \[
    \inner{F}{G} 
    =  \int_{x\in X} \tilde{f}(x)g(x){\mathrm d}\mu(x) 
    =  \int_{x\in X} {f}(x)\tilde{g}(x){\mathrm d}\mu(x)
    \,,
    \]
    with $\Tilde{f}=\hinv \h f$ and $\Tilde{g}=\hinv \h g$ the back-transformed functions.
\end{theorem}

\begin{proof}
    Using Proposition~\ref{commutationTint}, we have
    \arraycolsep=2pt
    \begin{eqnarray*}
    \inner{F}{G} &=& \inner{F}{ \int_{x\in X} g(x) \Delta^{\varphi}(x) {\mathrm d}\mu(x)}\\[.2cm]
    &=& \int_{x\in X} g(x) \inner{F}{\Delta^{\varphi}(x)} {\mathrm d}\mu(x)\\[.2cm]
    &=& \int_{x\in X} g(x)\tilde{f}(x) {\mathrm d}\mu(x)\,.
  \end{eqnarray*}
\end{proof}

The inner product between functions in $L^2(X)$ corresponds to the Euclidean inner product in $\mathbb{R}^D$.

\begin{corollary}
    Let $f$ be a function in $L^2(X)$ and $F=\h f$ its hyperdimensional transform, and let $1_X$ be the simple function mapping all elements of $X$ to $1$ and $\mathbb{1}_{X} = \h 1_X$ its hyperdimensional transform, then
    \[
    \inner{F}{\mathbb{1}_{X}} 
    = \int_{x\in X} f(x) \Tilde{1}_X(x) {\mathrm d}\mu(x)
     = \int_{x\in X} \Tilde{f}(x) {\mathrm d}\mu(x)
    \]
    and
    \[
    \lim_{D\rightarrow\infty}\inner{F}{\mathbb{1}_{X}} = \int_{x\in X} f(x){\mathrm d}\mu(x)
    \,.
    \]
\end{corollary}

\begin{proof}
    The first claim follows directly from Theorem~\ref{innerproduct} by setting $G=\mathbb{1}_{X}$. The second claim follows from the fact that for $D\rightarrow\infty$, the function $\Tilde{1}_X$ approximates $1_X$ perfectly, expressing the normalization of $\Delta^{\Phi}$ (see Remark~\ref{function1exact}).
    Note that, for $D\rightarrow\infty$, we thus also have that $ \int_{x\in X} f(x)d\mu(x) =  \int_{x\in X} \Tilde{f}(x)d\mu(x)$,
    {\em i.e.}, independently of the length scale, the smoothed function $\tilde{f}$ yields the same integral.
\end{proof}

\begin{corollary}
    Let $f$ be a function in $L^2(X)$ and $F=\h f$ its hyperdimensional transform, and let $1_A : X\rightarrow \{0,1\}$ be the simple function mapping all elements of a measurable subset $A \subset X$ to $1$ and all other elements of $X$ to $0$, and $\mathbb{1}_{A} = \h 1_A$ its hyperdimensional transform, then
    \[
    \inner{F}{\mathbb{1}_A} = \int_{x\in A} \tilde{f}(x){\mathrm d}\mu(x)
    \,.
    \]
\end{corollary}

Next, we also introduce the representation of the derivative of a function in the hyperdimensional space. Therefore, we add the assumption that $X\subset\mathbb{R}$ is a real interval, consider the metric $d(x,x')=|x-x'|$ and use the standard definition of the derivative. 

\begin{definition}
    Let $\D : X \rightarrow \mathbb{R}^D$ be a normalized hyperdimensional encoding of $X$ of which the components $\Delta^{\varphi}_i$ are functions that are $n$ times differentiable at $x \in X$. Then we say that $\D$ is $n$ times differentiable at $x$ and the $n$-th order derivative at $x$ is elementwisely given by
    \[
    \Delta^{\varphi, (n)}(x) := \left[ \frac{{\mathrm d}^n}{{\mathrm d}x^n}\Delta^{\varphi}_1(x), \frac{{\mathrm d}^n}{{\mathrm d}x^n}\Delta^{\varphi}_2(x), \ldots , \frac{{\mathrm d}^n}{{\mathrm d}x^n}\Delta^{\varphi}_D(x)  \right]
    \,.
    \]
\end{definition}

\begin{theorem}
\label{th:derivative}
     Let $\D : X \rightarrow \mathbb{R}^D$ be a normalized hyperdimensional encoding of $X$
    that is $n$ times differentiable at $x$ and $\Tilde{f} = \hinv \h f$ the back-transformed function of $f\in L^2(X)$, then $\Tilde{f}$ is also $n$ times differentiable at $x$ and
    \[
    \frac{{\mathrm d}^n}{{\mathrm d}x^n}\tilde{f}(x) = \inner{F}{\Delta^{\varphi, (n)}(x)}
    \,,
    \]
    with $F=\h f$.
\end{theorem}

\begin{proof}
    Because of the linearity of the inner product, we have
    \arraycolsep=2pt
    \begin{eqnarray*}
     \frac{{\mathrm d}}{{\mathrm d}x}\Tilde{f}(x) &=& \lim_{|h|\rightarrow 0} \frac{\Tilde{f}(x+h)-\Tilde{f}(x)}{h}\\
   &=& \lim_{|h|\rightarrow 0} \frac{\inner{F}{\Delta^{\varphi}(x+h)}-\inner{F}{\Delta^{\varphi}(x)}}{h}\\
    &=& \lim_{|h|\rightarrow 0} \inner{F}{\frac{\Delta^{\varphi}(x+h) - \Delta^{\varphi}(x)}{h}}\\
    &=& \inner{F}{\Delta^{\varphi, (1)}(x)}
    \,.
    \end{eqnarray*}
    The higher-order derivatives follow from recursion.
    \end{proof}

\begin{remark}
    Note that $\frac{{\mathrm d}^n}{{\mathrm d}x^n}\tilde{f}(x) = \inner{F}{\Delta^{\varphi, (n)}(x)}$ expresses the linear functionals of function evaluation and derivative function evaluation of any order in a unified way as explicit inner products with $F\in\mathbb{R}^D$. Similarly, the integral operator is a linear functional that is explicitly represented as a dot product with $F$,  i.e.,~$\int_{x\in X} \Tilde{f}(x) d\mu(x)=\inner{F}{\mathbb{1}_{X}}$.
\end{remark}

Typically, in the context of hyperdimensional computing, the hyperdimensional representation is low-memory. Consider Example~\ref{intervalencoding}, where the stochastic process $\Phi$ takes values in $S=\{-1,1\}$ and the unnormalized hyperdimensional representation $\varphi(x)\in \{-1,1\}^D$ can be represented as a bit vector.
In that case, each component $\varphi_i$ of $~\varphi$, being a random function switching between 1 and -1 at a certain frequency (see Example~\ref{intervalencoding}), is thus not differentiable.
In practice, this is not necessarily a limitation for estimating derivatives. Note that the encoding assumes a finite length scale $l>0$ within which point representations $\varphi(x)$ and $\varphi(x')$ are correlated. One can argue that the location of $x$ is thus fuzzy w.r.t.~a precision $l$. 
Consequently, one can argue that it is reasonable to approximate the derivative with a finite difference $h$ close to the length scale $l$.
The finite-difference derivative of $\tilde{f}$ as an approximation of the true derivative can be exactly computed via the finite-difference derivative of the encoding. The proof is analogous to that of Theorem~\ref{th:derivative}.

In Figure~\ref{fig:derivativeencoding}(a), a few lower-order derivatives of the step function computed with a finite difference are shown. The step function illustrates a component $\varphi_i$ of the unnormalized encoding that switches between 1 and -1 at a certain frequency. As an alternative, one may replace the step function in $\varphi_i$ with smooth alternatives based on, e.g., the sigmoid function (see Figure~\ref{fig:derivativeencoding}(b)). 
The latter approach results in a smoother function recovery $\Tilde{f}$ and an exact derivative expression, however, at the cost of a more complex encoding compared to the simple $\{1,-1\}$-encoding.
\begin{figure}
    \centering
    \includegraphics[width=0.48\textwidth]{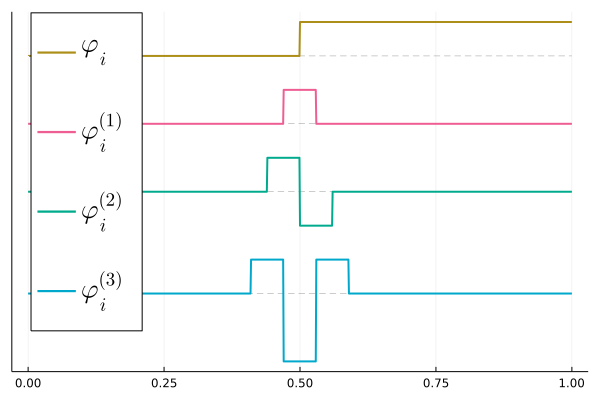}
    \includegraphics[width=0.48\textwidth]{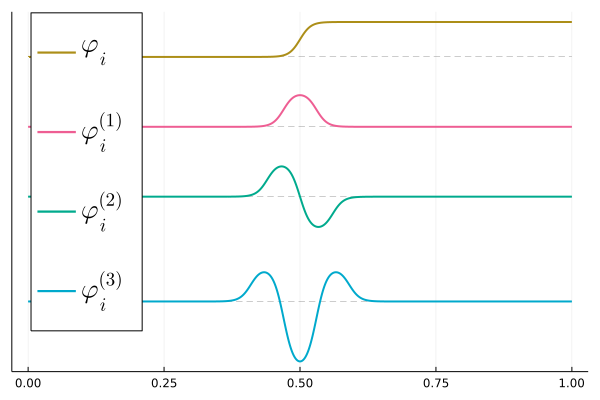}
    \caption{Lower-order derivatives of the step function as an illustration for the lower-order derivatives of a component $\varphi_i$ of the unnormalized encoding $\varphi$. As in Example~\ref{innerproduct}, $\varphi_i$ is a function switching between 1 and -1 at a certain frequency. In the left panel, a (centered) finite-difference derivative is computed. 
    In the right panel, a smooth differentiable alternative is used for the step function based on the sigmoid function. The functions are rescaled on the y-axis for proper visualization.}
    \label{fig:derivativeencoding}
\end{figure}

\section{Extensions to functions with multiple variables}
\label{sec:multiplevariables}
In this section, we extend the hyperdimensional transform to multivariate functions. We consider bivariate functions $f(x,y)$ with $x\in X$ and $y\in Y$; generalization to three or more variables is analogous and is not explicitly written down for the sake of brevity. 
Typical examples concern $X=Y=\mathbb{R}$, although one can envisage more involved settings. The only requirement is that the variables live in measure spaces.
The following assumptions stand throughout this section: $(X, \mathcal{A}, \nu)$ and $(Y, \mathcal{B}, \xi)$ are finite measure spaces; $\{\Phi(x) \mid x \in X\}$ and $\{\Psi(y) \mid y \in Y\}$ are stochastic processes taking values in bounded sets $S \subset \mathbb{R}$ and  $T \subset \mathbb{R}$, resp.; and $\Delta^{\varphi} : X \rightarrow \mathbb{R}^D$ and $\Delta^{\psi}: Y \rightarrow \mathbb{R}^D$ are normalized hyperdimensional encodings of $X$ and $Y$ w.r. t.~the stochastic processes $\Phi$ and $\Psi$. 
In the case that $X=Y$, the stochastic processes may be the same, but the sampled functions $\varphi$ and $\psi$ are always independent. Additionally, we assume that the normalized stochastic processes are zero-centered, {\em i.e.},~$\mathbb{E}\left[ \Delta^{\Phi}(x) \right] = \mathbb{E}\left[ \Delta^{\Psi}(y) \right] = 0$, for all $x \in X$ and $y\in Y$.

\subsection{Hyperdimensional representation: product encoding}
First, we introduce the hyperdimensional representation of a product space.

\begin{definition}
    The function $\Delta^{\varphi,\psi} : X\times Y \rightarrow \mathbb{R}^D$, given by
    \[
    \Delta^{\varphi,\psi}(x,y) = \Delta^{\varphi}(x) \otimes \Delta^{\psi}(y)
    \,,
    \]
     is called the product encoding of $\Delta^{\varphi}$ and $\Delta^{\psi}$. 
     Here, $\otimes$ denotes the elementwise product, i.e.,
    \[
    \Delta^{\varphi,\psi}_i(x,y) = \Delta^{\varphi}_i(x) \Delta^{\psi}_i(y)
    \,.
    \]
\end{definition}

Note that the order of $\varphi$ and $\psi$ in the notation $\Delta^{\varphi, \psi}$ is important: $\Delta^{\varphi,\psi}(x,y) = \Delta^{\varphi}(x) \otimes \Delta^{\psi}(y)$, while $\Delta^{\psi, \varphi}(x,y) = \Delta^{\psi}(x) \otimes \Delta^{\varphi}(y)$.

\begin{remark}
    In the limit of $D \rightarrow \infty$, we have
    \arraycolsep=2pt
    \begin{eqnarray*}
    \inner{\Delta^{\varphi}(x) \otimes \Delta^{\psi}(y)}{\Delta^{\varphi}( x' ) \otimes \Delta^{\psi}(y')} =  \inner{\Delta^{\varphi}(x) }{\Delta^{\varphi}(x')}
    \inner{\Delta^{\psi}(y)}{ \Delta^{\psi}(y')}    \,,
    \end{eqnarray*}
    which is a basic result from statistics on the covariance of products of zero-centered random variables. 
    This general outer product (or tensor product) property motivates the use of $\otimes$ for denoting the elementwise product. The property holds for infinite dimensionality and holds only approximately for finite dimensionality. The advantage of this approximation is that the dimensionality $D$ is a constant, whereas the dimensionality of a real outer product increases as $D^2$.
\end{remark}

Consider the product measure space $(X \times Y, \mathcal{A} \times \mathcal{B}, \mu)$. Here, 
$\mathcal{A} \times \mathcal{B}$ is the $\sigma$-algebra generated by the Cartesian products of elements 
of $\mathcal{A}$ and $\mathcal{B}$. The product measure $\mu$ is uniquely determined as $\mu(A \times B)=\nu(A)\xi(B)$, for any $A \in \mathcal{A}$ and $B \in \mathcal{B}$, if both measure 
spaces are $\sigma$-finite, which is a standard assumption (a finite measure space is also $\sigma$-finite).

With this product measure on the product space $X\times Y$, and the product encoding $\Delta^{\varphi, \psi} : X\times Y \rightarrow \mathbb{R}^D$, the hyperdimensional transform of $f\in L^2(X\times Y)$ takes the form
\[
 F 
    =\mathcal{H}^{\Delta^{\varphi, \psi}}f
    = \int_{X \times Y} f(x,y)\Delta^{\varphi,\psi}(x,y) {\mathrm d}\mu(x,y)
    \,.
    \]
According to Fubini's theorem for product measures, this integral can be computed using iterated integrals and the order of integration can be changed, {\em i.e.},
    \[
    \begin{split}
    F &= \int_{Y} \left( \int_{X} f(x,y)\Delta^{\varphi,\psi}(x,y) {\mathrm d}\nu(x) \right) {\mathrm d}\xi(y)\\
    &= \int_{X} \left( \int_{Y} f(x,y)\Delta^{\varphi,\psi}(x,y) {\mathrm d}\xi(y) \right) {\mathrm d}\nu(x)
    \,.
    \end{split}
    \]

As a product measure space is again a measure space itself, the aforementioned theory on the hyperdimensional transform, inverse transform and approximation properties, is still applicable. In what follows, we add some additional results that apply in particular to product measure spaces.

\subsection{Marginalisation}\label{sec:condint}
As a second extension for multiple variables, we describe how one can integrate a single variable while fixing the others.

\begin{theorem}
\label{th:conditinalintegrals}
    Let $1_Y : Y \rightarrow\{1\}$ be the simple function mapping all elements of $Y$ to 1 and $\mathbb{1}_{Y} = \mathcal{H}^{\Delta^{\psi}} 1_Y$ its hyperdimensional transform. Let $f : X \times Y \rightarrow \mathbb{R}$ be a bivariate function in $L^2(X \times Y)$, $F=\mathcal{H}^{\Delta^{\varphi,\psi}} f$ its hyperdimensional transform, and $\Tilde{f}=\tilde{\mathcal{H}}^{\Delta^{\varphi,\psi}} \mathcal{H}^{\Delta^{\varphi,\psi}} f$ the back-transformed function, then
    \arraycolsep=2pt
    \begin{eqnarray}
     \int_{Y} \Tilde{f}(x,y){\mathrm d}\xi(y)
        &=&  \inner{F}{\Delta^{\varphi}(x)\otimes \mathbb{1}_{Y}} \label{eq:first} \\ 
        &=& \inner{F \otimes \Delta^{\varphi}(x)}{\mathbb{1}_{Y}} \label{eq:second} \\
        &=&  \inner{F \otimes \mathbb{1}_{Y}}{\Delta^{\varphi}(x)} \label{eq:third}
        \,.
        \end{eqnarray}
    \end{theorem}

    \begin{proof} 
    We prove the first equality, the other ones being analogous. Filling in the expressions, and using Proposition~\ref{commutationTint} and Fubini's theorem, we have
    \arraycolsep=2pt
    \begin{eqnarray*}
        \inner{F}{\Delta^{\varphi}(x)\otimes \mathbb{1}_{Y}} &=& \left\langle{\int_{X} \int_{Y} f(x',y')\Delta^{\varphi}(x')\otimes\Delta^{\psi}(y'){\mathrm d}\nu(x'){\mathrm d}\xi(y')}\right.  \left. ,{\Delta^{\varphi}(x) \otimes \int_{Y} \Delta^{\psi}(y) {\mathrm d}\xi(y)}\right\rangle\\
        &=& \int_{Y} \left[ \int_{X} \int_{Y} f(x',y') \inner{\Delta^{\varphi,\psi}(x',y')}{\Delta^{\varphi,\psi}(x,y)}\right. \left. {\mathrm d}\nu(x'){\mathrm d}\xi(y') \right] {\mathrm d}\xi(y) \\
        &=& \int_{\mathcal{Y}} \Tilde{f}(x,y) {\mathrm d}\xi(y)
        \,.
    \end{eqnarray*}
    \end{proof}

\begin{remark}
    The three expressions in Theorem~\ref{th:conditinalintegrals} have particular interpretations, which one might interpret as a basis for Bayesian inference with complex distributions:
   
    Eq.~(\ref{eq:first}): Using the extension of the hyperdimensional transform for measures, one can interpret the expression $\Delta^{\varphi}(x)\otimes \mathbb{1}_{Y}$ as $\mathcal{H}_*^{\Delta^{\varphi}}\delta_x \otimes \mathcal{H}^{\Delta^{\psi}}1_Y$. The inner product thus represents the evaluation of a function that is Dirac-distributed in the variable $x$ and has a constant density $1$ in the variable $y$.
     
    Eq.~(\ref{eq:second}): The expression $F \otimes \Delta^{\varphi}(x)$ can be seen as the representation of a univariate function in the variable $y$, conditioned on $x$. This univariate function is integrated w.r.t.\ the variable $y$ by the inner product with $\mathbb{1}_Y$.
       
    Eq.~(\ref{eq:third}): The expression $F \otimes \mathbb{1}_{Y}$ can be seen as a marginal univariate function in the variable $x$. The inner product with $\D(x)$ is then simply a function evaluation of this function at $x$. Marginalizing a multivariate function in hyperdimensional space thus simply corresponds to an elementwise vector multiplication.

\end{remark}

\subsection{Partial derivatives and gradients}
For a last extension for multiple variables, we add the assumption that $X\subset\mathbb{R}$ and $Y\subset\mathbb{R}$ are real intervals and we use the standard definition of the (partial) derivative. 

\begin{theorem}
Let $\Delta^{\varphi}$ and $\Delta^{\psi}$ be normalized hyperdimensional encodings of $X$ and $Y$ that are differentiable at $(x,y)\in X\times Y$, and let $f : X \times Y \rightarrow \mathbb{R}$ be a bivariate function in $L^2(X \times Y)$, $F=\mathcal{H}^{\Delta^{\varphi,\psi}} f$ its hyperdimensional transform, and $\tilde{f} = \tilde{\mathcal{H}}^{\Delta^{\varphi,\psi}} 
 \mathcal{H}^{\Delta^{\varphi,\psi}} f$ the back-transformed function, then the gradient of $\Tilde{F}$ at $(x,y)$ exists and
\[
\frac{\partial}{\partial x}\Tilde{f}(x,y) = \inner{F}{\Delta^{\varphi, (1)}(x)\otimes \Delta^{\psi}(y)}
\]
and
\[
\frac{\partial}{\partial y}\Tilde{f}(x,y) = \inner{F}{\Delta^{\varphi}(x)\otimes \Delta^{\psi, (1)}(y)}
\,.
\]
\end{theorem}
\begin{proof}
    The proof is analogous to that of Theorem~\ref{th:derivative}.
\end{proof}

\section{Application: expressing linear differential and integral equations}
\label{sec:differentialandintegralequations}
In this section, we illustrate how the functionals of function evaluation, derivative function evaluation, and integral evaluation as explicit inner products in hyperdimensional space naturally allow for expressing linear differential and integral equations. Contrary to solving differential equations via other integral transforms (e.g., Laplace or Fourier), no analytical expressions for the transform or the inverse transform are required. Instead, the hyperdimensional transform offers a more numerical approach, where the infinite-dimensional function is approximated by a vector of finite, large dimensionality~$D$. 
This approach unifies solving differential equations and performing linear regression, thus establishing a connection with the fields of statistical modelling and machine learning.

We retain the standard assumptions from Section~\ref{sec:transform}, and additionally assume that $X \subset \mathbb{R}$ is a real interval.

\subsection{Linear differential equations}
Consider the general form of a linear differential equation for $x\in [a,b]$:
\[
    a_0(x)f(x) + a_1(x)\frac{{\mathrm d}}{{\mathrm d}x}f(x)
    + \cdots + a_n(x)\frac{{\mathrm d}^n}{{\mathrm d}x^n}f(x) = b(x)
    \,.
\]
Approximating the solution as $\Tilde{f}=\hinv F$ with $F\in \mathbb{R}^D$ and using $\frac{{\mathrm d}^n}{{\mathrm d}x^n}\tilde{f}(x) = \inner{F}{\Delta^{\varphi, (n)}(x)}$, the differential equation takes the following form:

\begin{eqnarray*}
    a_0(x)\inner{F}{\Delta^{\varphi}(x)} + a_1(x)\inner{F}{\Delta^{\varphi, (1)}(x)}  + \cdots + a_n(x)\inner{F}{\Delta^{\varphi, (n)}(x)} = b(x)\,.
\end{eqnarray*}
Equivalently, we have:
\[
\inner{F}{a_0(x)\Delta^{\varphi}(x) + a_1(x)\Delta^{\varphi,(1)}(x) + \cdots + a_n(x)\Delta^{\varphi,(n)}(x)} = b(x)
\,.
\]
Imposing that the differential equation must hold at points $x^i$, $i=1,\ldots,m$,
leads to a system of $m$ equations:
\begin{equation}
\label{eq:linearregression}
XF=B
\,.
\end{equation}
The $m \times D$ matrix $X$ stacks the representations of the right-hand sides of the inner products at each of the $m$ points, {\em i.e.},
\[
X =
\begin{bmatrix}
\left[ a_0(x^1)\Delta^{\varphi}(x^1) + a_1(x^1)\Delta^{\varphi,(1)}(x^1) + \cdots + a_n(x^1)\Delta^{\varphi,(n)}(x^1) \right]^\mathrm{T} \\[.2cm]
\left[ a_0(x^2)\Delta^{\varphi}(x^2) + a_1(x^2)\Delta^{\varphi,(1)}(x^2) + \cdots + a_n(x^2)\Delta^{\varphi,(n)}(x^2) \right]^\mathrm{T} \\[.2cm]
\vdots \\[.2cm]
\left[ a_0(x^m)\Delta^{\varphi}(x^m) + a_1(x^m)\Delta^{\varphi,(1)}(x^m) + \cdots + a_n(x^m)\Delta^{\varphi,(n)}(x^m) \right]^\mathrm{T} \\
\end{bmatrix}
\,.
\]
The $D$-dimensional vector $B$ stacks the function evaluations of $b$ at these points, {\em i.e.},
\[
B = 
\begin{bmatrix}
b(x^1) \\[.2cm]
b(x^2) \\[.2cm]
\vdots \\[.2cm]
b(x^m) \\
\end{bmatrix}
\,.
\]
System~\eqref{eq:linearregression} imposes the differential equation at the points $x^i$, $i=1,\ldots,m$. If the maximal distance between the $m$ points is not greater than the length scale of the encoding, then there are no points $x\in [a,b]$ such that $\inner{\D(x)}{\D(x^i)}=0$.  Due to the correlation within the length scale $l$, the entire domain is thus taken into account in the system corresponding to the $m$ points.

Note that this system takes the exact same form of a regular linear regression problem with $F$ the model parameters searched for and $X$ the matrix with highly nonlinear features. With an ordinary least squares assumption and a small regularization term for numerical stability, such a problem is typically solved either via the exact solution of ridge regression or via iterative methods such as conjugate gradient descent. The main difference with linear regression is that the $m$ data points are now not just $m$ simple $0$-th order function observations in the form of $f(x^i)=b_i$. Instead, each data point can now express a more complex higher-order function observation, e.g., $f(x^i)+ f'(x^i) = b_i$. Also, boundary conditions of any order can be added to the equations as they can also be expressed as inner products and take the very same form. Adding the equations for the boundary conditions, the system of linear equations that expresses the differential equation and the boundary conditions can be written as $X_c F = B_c$.
We refer to Figure~\ref{fig:diffequation} for some examples. Note that, with this finite dimensionality, the result almost seems not noisy at all.
Here, $\Tilde{f}$ was optimized to match the differential equation as good as possible. The conditions on the derivatives of $\Tilde {f}$ and the ridge regularization may ensure a smoother $\Tilde{f}$.

\begin{figure}
        \centering        
        \includegraphics[width=0.65\textwidth]{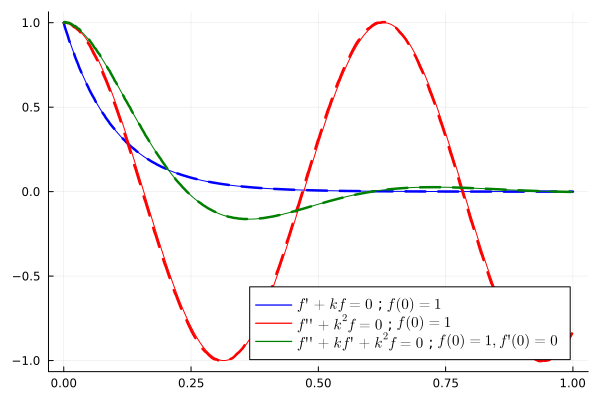}
        \caption{Hyperdimensional solutions of basic differential equations describing exponential decay, the harmonic oscillator, and the damped harmonic oscillator. 
        Both the hyperdimensional solutions (dashed line), taking into account the boundary conditions, and the analytical solutions are plotted (solid line). We set $k=10$. The encoding of Example~\ref{intervalencoding}, a dimensionality $D=5000$, a length scale $l=0.05$, a finite difference $h=l/5 $ for derivation, and 500 equidistant points $x^i \in [0,1]$ were used. The solution was computed using the exact solution of ridge regression~\cite{mcdonald2009ridge}, with a small ridge regression parameter $\lambda=1$ for numerical stability.}
        \label{fig:diffequation}
    \end{figure}


\begin{remark}
    The above approach unifies performing linear regression and solving a differential equation. When no data points for regression are given, the differential equation will dominate, and vice versa. 
    One may also consider linear regression as the main task, and see the differential equation as a kind of regularization. For example, when little or no data points for regression are available, one may impose the function to behave as a constant, linearly, quadratically, etc., by imposing $f'(x)=0$, $f''(x)=0$, $f'''(x)=0$, etc.
\end{remark}

Solving a differential equation in hyperdimensional space takes this simple form because the function is represented as a vector $F$, and the functionals that query a function evaluation, a derivative function evaluation, etc., are all represented as inner products with $F$. A system of $m$ linear equations in the components of $F$ 
can then simply be constructed by expressing at which $m$ points the equations must hold. 

\subsection{Linear integral equations}
The very same reasoning holds for integral equations. Next, we show how solving an integral equation can be turned into solving a linear regression problem.
A prominent example of a nonlinear integral equation is the Fredholm equation of the second type:
\[
f(x) = b(x) + \lambda \int_c^d k(y,x) f(y) {\mathrm d}y
\,.
\]
The functions $b : [c,d]\rightarrow \mathbb{R}$ and $k: [c,d] \times [c,d] \rightarrow \mathbb{R}$ are given, $\lambda$ is a constant, and $f:[c,d]\rightarrow \mathbb{R}$ is the function searched for. We approximate the solution as $\Tilde{f}=\h f$ and 
use $K=\mathcal{H}^{\Delta^{\varphi, \psi}} k$ as the hyperdimensional representation of $k$. 
Recall that the multivariate transform $K=\mathcal{H}^{\Delta^{\varphi, \psi}} k$ w.r.t.~the
hyperdimensional encoding $\Delta^{\varphi, \psi}$ assumes that the first variable (here, $y$) and the second variable (here, $x$) are independently encoded by $\D$ and $\Delta^{\psi}$, respectively.
The integral equation can be written as
\begin{equation}
\label{eq:integralequation}
\inner{F}{\Delta^{\varphi}(x)} = b(x) + \lambda \inner{F}{K \otimes \Delta^{\psi}(x)}
\,.
\end{equation}
Note that the inner product on the right-hand side of Eq.~(\ref{eq:integralequation}) integrates the variable~$y$ encoded by $\D$, while the variable~$x$, encoded by $\Delta^{\psi}$, is used for conditioning $k(y, x)$ (see marginalisation of multivariate functions in Section~\ref{sec:condint}).
The equation can be rewritten as a single inner-product equation 
\[
\inner{F}{\Delta^{\varphi}(x) - \lambda K \otimes \Delta^{\psi}(x)} = b(x)
\,,
\]
such that a linear regression matrix equation is again obtained by choosing points~$x^i$, $i=1,\ldots,m$.

\section{Connections with other integral transforms}
\label{sec:othermethodsandfuturedirections}
In this section, we first relate the hyperdimensional transform to other integral transforms in general, focusing on prominent examples such as the transforms of Laplace and Fourier. Second, we discuss the close connection with the fuzzy transform in greater detail.

\subsection{Integral transforms}
As introduced in Section~\ref{sec:intro}, the hyperdimensional transform is an integral transform just like the Laplace transform, the Fourier transform and the fuzzy transform. While the Laplace and Fourier transforms yield functions of complex or real variables, the hyperdimensional transform and the fuzzy transform yield functions with as domain a finite set.
Vectorizing the function values, the fuzzy and hyperdimensional transform can be interpreted as function-to-vector transformations. 

On the one hand, the finite dimensionality $D$ of the hyperdimensional transform might imply less expressivity and entail some loss of information, while the random nature of the basis functions introduces stochastic noise. However, these effects diminish with increasing dimensionality $D$ of the vector. Hence, the dimensionality is assumed large.

On the other hand, the transformation to a finite-dimensional vector makes the computation of the integral tractable for a broader set of functions: each component of the transform can be computed directly without the need for an analytical expression. Note that the hyperdimensional transform is defined for any abstract universe $X$ that is provided with a measure, allowing,  e.g., for representing functions on sets, sequences, or graphs.

The hyperdimensional transform opens a distinct approach to solving differential equations. Instead of an analytical solution, an approximate solution can be computed. Thanks to the natural expressions of the functionals that include differentiation and integration, the hyperdimensional transform converts linear differential equations and linear integral equations into linear matrix equations, unifying them with linear regression.

While the Fourier transform decomposes a function in an infinite set of wave functions of all possible frequencies, the hyperdimensional transform decomposes a function in random wave-like functions. For instance, in Example~\ref{intervalencoding}, these wave-like functions randomly switch between $1$ and $\textrm{-}1$ at some `average frequency' that is related to the length scale $l$. Due to the possibility of setting a finite length scale~$l$, the hyperdimensional transform allows for incorporating noisy data. 
Similarly, for the fuzzy transform (see Section~\ref{sec:fuzzytransform}), the lower expressivity due to the finite dimension and a notion of length scale allow for filtering noise.
Also, approaches in hyperdimensional computing based on holographic representations allow for noise-robust classification in machine learning. 

\subsection{The fuzzy transform}
\label{sec:fuzzytransform}
Because of its close connection with the hyperdimensional transform, we discuss the fuzzy transform in more detail. For a comprehensive overview of the fuzzy transform, we refer to~\cite{perfilieva2006fuzzy}. 

Let $[a,b]\subset \mathbb{R}$ be an interval and $x_1<\cdots<x_D$ fixed nodes such that $x_1=a$ and $x_D=b$. 
A set of basis functions $A_s: [a,b] \to [0,1]$, $s=1,\ldots,D$, is called  a \textit{fuzzy partition} of $[a,b]$
if the following conditions are satisfied, for $s=1,\ldots,D$:
\begin{itemize}
    \item[(i)] $A_s(x_s)=1$; 
    \item[(ii)] $A_s(x)=0$ if $x \notin\, ]x_{s-1},x_{s+1}[$ (with the convention $x_0=a$ and $x_{D+1}=b$);
    \item[(iii)] $A_s(x)$ is continuous;
    \item[(iv)] $A_s(x)$ strictly increases on $[x_{s-1}, x_s]$ and strictly decreases on $[x_s, x_{s+1}]$.
\end{itemize}
A prominent example is the set of uniform triangular basis functions given by $A_s(x) = \max \left(0, 1 - \frac{|x-x_s|} {\lambda} \right)$ with $\lambda=(b-a)/(D-1)$, illustrated in Figure~\ref{fig:fuzzybasic} with $[a,b]=[0,5]$ and $D=6$. Note the exact correspondence with the expression in Example~\ref{intervalencoding}, {\em i.e.},~$A_s(x) = \inner{\Phi(x)}{\Phi(x_s)}$.
\begin{figure}
    \centering
    \includegraphics[width=0.45\textwidth]{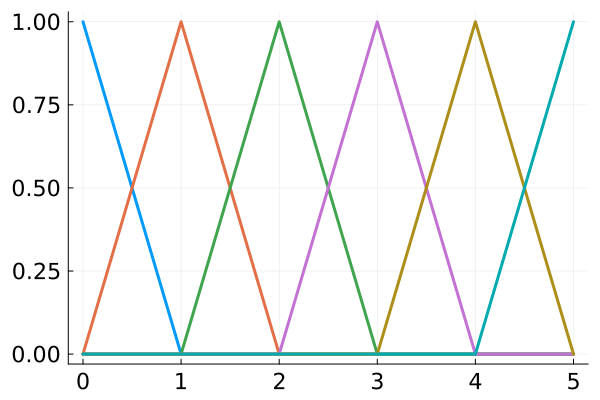}
    \caption{Example of a uniform triangular fuzzy partition of the interval [0,5] with six basis functions.}
    \label{fig:fuzzybasic}
\end{figure}

The components of the fuzzy transform $G$ of a function $g:[a,b]\to \mathbb{R}$ are given by
\begin{equation}
\label{eq:fuzzytransform}
G_s := \frac{\int_a^b g(x) A_s(x) {\mathrm d}x}{\int_a^b A_s(x) {\mathrm d}x}
\,.
\end{equation}
Each $s$-th component can thus be interpreted as a local weighted mean of the function around the node $x_s$. 
The back-transformed function is then given by
\begin{equation}
\label{eq:inversefuzzytransform}
\Tilde{g}(x) :=  \sum_{s=1}^{D} G_s A_s(x)\,.
\end{equation}
Following the definition of the functions $A_s$, the function $\Tilde{g}$ evaluated at node $x_s$ equals the component $G_s$, {\em i.e.},
\begin{eqnarray}
    \Tilde{g}(x_s) &=& G_s \\ &=& \frac{\int_a^b g(x) \nonumber A_s(x) {\mathrm d}x}{\int_a^b A_s(x) {\mathrm d}x} \label{eq:fuzzynodeevaluation} \\[.2cm]
   &=& \frac{\int_a^b g(x) \inner{\Phi(x)}{\Phi(x_s)} {\mathrm d}x}{\int_a^b \inner{\Phi(x)}{\Phi(x_s)} {\mathrm d}x}
    \,.
\end{eqnarray}
For the hyperdimensional transform, recall that
\begin{equation}
\label{eq:transform}
\Tilde{f}(x_s) =  \int_{x \in [a,b]} f(x) \frac{\inner{\Phi(x)}{\Phi(x_s)}}{n(x)n(x')} {\mathrm d}\mu(x)
\,,
\end{equation}
which is equivalent to Eq.~(\ref{eq:fuzzynodeevaluation}) if the normalization function $n$ can be determined as 
\[
n(x) = \sqrt{\int_{x \in [a,b]} \inner{\Phi(x)}{\Phi(x_s)} {\mathrm d}\mu(x)}
\]
and is constant. This is the case in Example~\ref{exampleperiodic} without boundaries and in Example~\ref{intervalencoding} if one may neglect the boundary effects (e.g., when $\lambda$ is small). 
In general, the normalization function is not a constant. One main difference between the hyperdimensional transform and the fuzzy transform is thus the way of normalization.

A second main difference is that for the fuzzy transform, $\Tilde{f}(x)$ in Eq.~(\ref{eq:inversefuzzytransform}) interpolates between nodes $x_s$, linearly in the case of triangular basis functions, while for the hyperdimensional transform, no specific choice for the nodes is made and Eq.~(\ref{eq:transform}) holds at any point $x$ and not only at nodes $x_s$. The hyperdimensionally back-transformed function can be interpreted as a moving window average, instead of a (linear) interpolation between averages, however, possibly with some stochastic noise, depending on the dimensionality and the smoothness of the encoding.

Similarly as for the hyperdimensional transform, the fuzzy transform can be used to solve (partial) differential equations and to handle noisy data~\cite{perfilieva2001fuzzy,perfilieva2004fuzzy, stepnicka2005numerical}. Derivatives are computed based on a finite difference between the components $F_s$ of the fuzzy transform; for more details on solving differential equations with the fuzzy transform, we refer to~\cite{perfilieva2004fuzzy}. The hyperdimensional transform may either use a finite difference or an infinitesimal difference, depending on whether the encoding is differentiable.
Both methodologies can be seen as approximate approaches to solving the differential equation with some finite length scale/precision. 

\section{Conclusion}

We formally introduced the hyperdimensional transform, allowing for the approximation of functions by holographic, high-dimensional representations called hyperdimensional vectors. We discussed general transform-related properties such as the uniqueness of the transform, approximation properties of the inverse transform, and the representation of inner products, integrals, and derivatives. The hyperdimensional transform offers theoretical foundations and insights for research in the field of hyperdimensional computing.

We also demonstrated how this transform can be used to solve linear differential and integral equations and discussed the connection with other integral transforms, such as the Laplace transform, the Fourier transform, and the fuzzy transform. Due to its capabilities of handling noisy data, we also anticipate applications in the fields of machine learning and statistical modelling. In our future work, we will elaborate further in this direction. Obvious aspects include an empirical estimation of the transform based on a sample of function evaluations, and a bipolar approximation of the transform that harnesses, even more, the fast and efficient capabilities of hyperdimensional computing. Additionally, the capability of the transform to represent entire signals, functions, or distributions as points in hyperdimensional space opens up new possibilities.


\bibliographystyle{abbrv}
\bibliography{references}

\end{document}